\documentclass{article}

\pdfoutput=1
\PassOptionsToPackage{numbers, compress}{natbib}

\usepackage[preprint]{neurips_2024}

\usepackage[utf8]{inputenc} 
\usepackage[T1]{fontenc}    
\usepackage{booktabs}       
\usepackage{amsfonts}       
\usepackage{nicefrac}       
\usepackage{microtype}      
\usepackage{xcolor}         

\usepackage{mathtools}
\usepackage{amsthm}
\usepackage{algorithmic}
\usepackage{wrapfig}
\usepackage[capitalize,noabbrev]{cleveref}
\usepackage{graphicx}
\usepackage{subcaption}
\usepackage{multirow}
\usepackage[shortlabels]{enumitem}
\usepackage[textsize=tiny]{todonotes}
\usepackage{authblk}

\newcommand\preitem{\mdseries\textbullet\space}
\newlist{desclist}{description}{3}
\setlist[desclist,1]{format=\preitem\bfseries,leftmargin=\widthof{\preitem},style=sameline}

\DeclareMathOperator*{\argmin}{argmin}
\DeclareMathOperator*{\argmax}{argmax}
\DeclareMathOperator{\TV}{TV}
\DeclareMathOperator{\cossim}{sim}

\title{Fair Anomaly Detection For Imbalanced Groups}

\author[1]{Ziwei Wu*}
\author[1]{Lecheng Zheng*}
\author[1]{Yuancheng Yu}
\author[1]{Ruizhong Qiu}
\author[2]{John Birge}
\author[1]{Jingrui He}

\affil[1]{University of Illinois at Urbana-Champaign \authorcr
  \{\tt ziweiwu2,lecheng4,yyu51,rq5,jingrui\}@illinois.edu}
\affil[2]{University of Chicago \authorcr
  \{\tt john.birge\}@chicagobooth.edu}

\begin{document}

\newcommand{\theHalgorithm}{\arabic{algorithm}}
\theoremstyle{plain}
\newtheorem{theorem}{Theorem}[section]
\newtheorem{proposition}[theorem]{Proposition}
\newtheorem{lemma}[theorem]{Lemma}
\newtheorem{corollary}[theorem]{Corollary}
\theoremstyle{definition}
\newtheorem{definition}[theorem]{Definition}
\newtheorem{assumption}[theorem]{Assumption}
\theoremstyle{remark}
\newtheorem{remark}[theorem]{Remark}


\newcommand{\he}[1]{{\textcolor{red}{[From He: #1]}}}
\newcommand{\ziwei}[1]{{\textcolor{blue}{[From Ziwei: #1]}}}
\newcommand{\lc}[1]{{\textcolor{brown}{[From Lc: #1]}}}
\newcommand{\lf}[1]{{\textcolor{red}{[From Yuan: #1]}}}
\newcommand{\RZ}[1]{{\textcolor{teal}{[From Ruizhong: #1]}}}
\newcommand{\name}{\textsc{FairAD}}

\newcommand{\eg}{{\sl e.g.}}
\newcommand{\ie}{{\sl i.e.}}
\newcommand{\etc}{{\sl etc.}}
\newcommand{\etal}{{\sl et al.}}

\newcommand{\pp}{\mathcal{P}_P}
\newcommand{\pu}{\mathcal{P}_U}

\newcommand{\stat}[2]{\( #1 {\pm #2} \)}
\newcommand{\bfstat}[2]{ $\textbf{#1$\pm$#2} $}

\maketitle

\begin{abstract}
Anomaly detection (AD) has been widely studied for decades in many real-world applications, including fraud detection in finance, and intrusion detection for cybersecurity, etc. 
Due to the imbalanced nature between protected and unprotected groups and the imbalanced distributions of normal examples and anomalies, the learning objectives of most existing anomaly detection methods tend to solely concentrate on the dominating unprotected group.
Thus, it has been recognized by many researchers about the significance of ensuring model fairness in anomaly detection. 
However, the existing fair anomaly detection methods tend to erroneously label most normal examples from the protected group as anomalies in the imbalanced scenario where the unprotected group is more abundant than the protected group. This phenomenon is caused by the improper design of learning objectives, which statistically focus on learning the frequent patterns (i.e., the unprotected group) while overlooking the under-represented patterns (i.e., the protected group). To address these issues, we propose \name, a fairness-aware anomaly detection method targeting the imbalanced scenario. It consists of a fairness-aware contrastive learning module and a rebalancing autoencoder module to ensure fairness and handle the imbalanced data issue, respectively. Moreover, we provide the theoretical analysis that shows our proposed contrastive learning regularization guarantees group fairness. Empirical studies demonstrate the effectiveness and efficiency of \name\ across multiple real-world datasets. 
\end{abstract}

\section{Introduction}
Anomaly detection (AD), a.k.a. outlier detection, is referred to as the process of detecting data instances that significantly deviate from the majority of data instances \citep{chandola2009anomaly}.
Anomaly detection finds extensive use in a wide variety of applications including financial fraud detection~\cite{DBLP:journals/compsec/WestB16, DBLP:journals/access/HuangMYC18}, pathology analysis in the medical domain~\cite{DBLP:journals/bmcbi/FaustXHGVDD18, DBLP:journals/access/ShvetsovaBFSD21} and intrusion detection for cybersecurity~\cite{DBLP:journals/jnca/LiaoLLT13, DBLP:journals/ett/AhmadKSAA21}. For example, an anomalous traffic pattern in a computer network suggests that a hacked computer is sending out sensitive data to an unauthorized destination \cite{ahmed2016survey}; anomalies in credit card transaction data could indicate credit card or identity theft \citep{rezapour2019anomaly}.


Up until now, a large number of deep anomaly detection methods have been introduced, demonstrating significantly better performance than shallow anomaly detection in addressing challenging detection problems in a variety of real-world applications. For instance, \cite{DBLP:conf/iclr/SohnLYJP21,DBLP:conf/icml/LiCCWTZ23} aim to learn a scalar anomaly scoring function in an end-to-end fashion, while ~\cite{DBLP:conf/kdd/AudibertMGMZ20, DBLP:conf/icde/ChenDHZZZZ21, DBLP:conf/iccv/HouZZXPZ21, DBLP:conf/aaai/YanZXHH21, DBLP:conf/sdm/WangLMGGLW023} propose to learn the patterns for the normal examples via a feature extractor.

Recently, there has been widespread recognition within the AI community about the significance of ensuring model fairness and thus it is highly desirable to establish specific parity or preference constraints in the context of anomaly detection. Take racial bias in anomaly detection as an example. Racial bias has been observed in predictive risk modeling systems to predict the likelihood of future adverse outcomes in child welfare \cite{chouldechova2018case}.
Communities in poverty or specific racial or ethnic groups may face disadvantages due to the reliance on government administrative data. The data collected from these communities, often stemming from their economic status and welfare dependence, can inadvertently categorize them as high-risk anomalies, leading to more frequent investigations to these minority groups.
Consequently, disproportionately flagging minority groups as anomalies not only perpetuates biases but also results in an inefficient allocation of government resources. 

To mitigate potential bias in anomaly detection tasks, numerous researchers \cite{DBLP:conf/kdd/SongLL21, DBLP:conf/fat/ZhangD21, DBLP:conf/iccv/FioresiDS23} advocate for incorporating fairness constraints into their proposed methods.
However, most of these methods tend to erroneously label most normal examples from the protected/minority group as anomalies in an imbalanced data scenario where the unprotected group is more abundant than the protected group.
To better illustrate this issue, we conduct a toy example on the MNIST-USPS dataset\cite{DBLP:conf/fat/ZhangD21}. Figure~\ref{fig:motivation} and Table~\ref{tab:motivation} show the performance of anomaly detection methods evaluated on the MNIST-USPS dataset, where \textit{Recall Diff} refers to the absolute value of recall difference between the protected group and the unprotected group. Note that in such an imbalanced scenario, metrics such as accuracy difference~\cite{zafar2017fairness} are not proper choices. We observe that existing methods either compromise performance for fairness (i.e., low recall rate and low recall difference) or exhibit unfair behavior (i.e., high recall difference); 
The problem of misclassification arises from models focusing on learning frequent patterns in the more abundant unprotected group, potentially overlooking under-represented patterns in the protected group. The issue of group imbalance results in higher errors for protected groups, thus causing misclassifications. Following~\cite{hashimoto2018fairness}, we refer to this phenomenon as \emph{representation disparity}. 

To address these issues, we face the following two major challenges.
\textbf{C1: Handling imbalanced data.} Due to the imbalanced nature between the protected and unprotected groups and the imbalanced distributions of normal examples and anomalies, the learning objectives of most existing anomaly detection methods tend to solely concentrate on the unprotected group.  \textbf{C2: Mitigating the representation disparity.} Traditional anomaly detection methods encounter difficulties in dealing with representation disparity issues, which may worsen in the imbalanced data scenario as protected groups are typically less rich than unprotected groups. 
\begin{figure}[t]
  \begin{minipage}[b]{.48\linewidth}
    \centering
    \includegraphics[width=\linewidth]{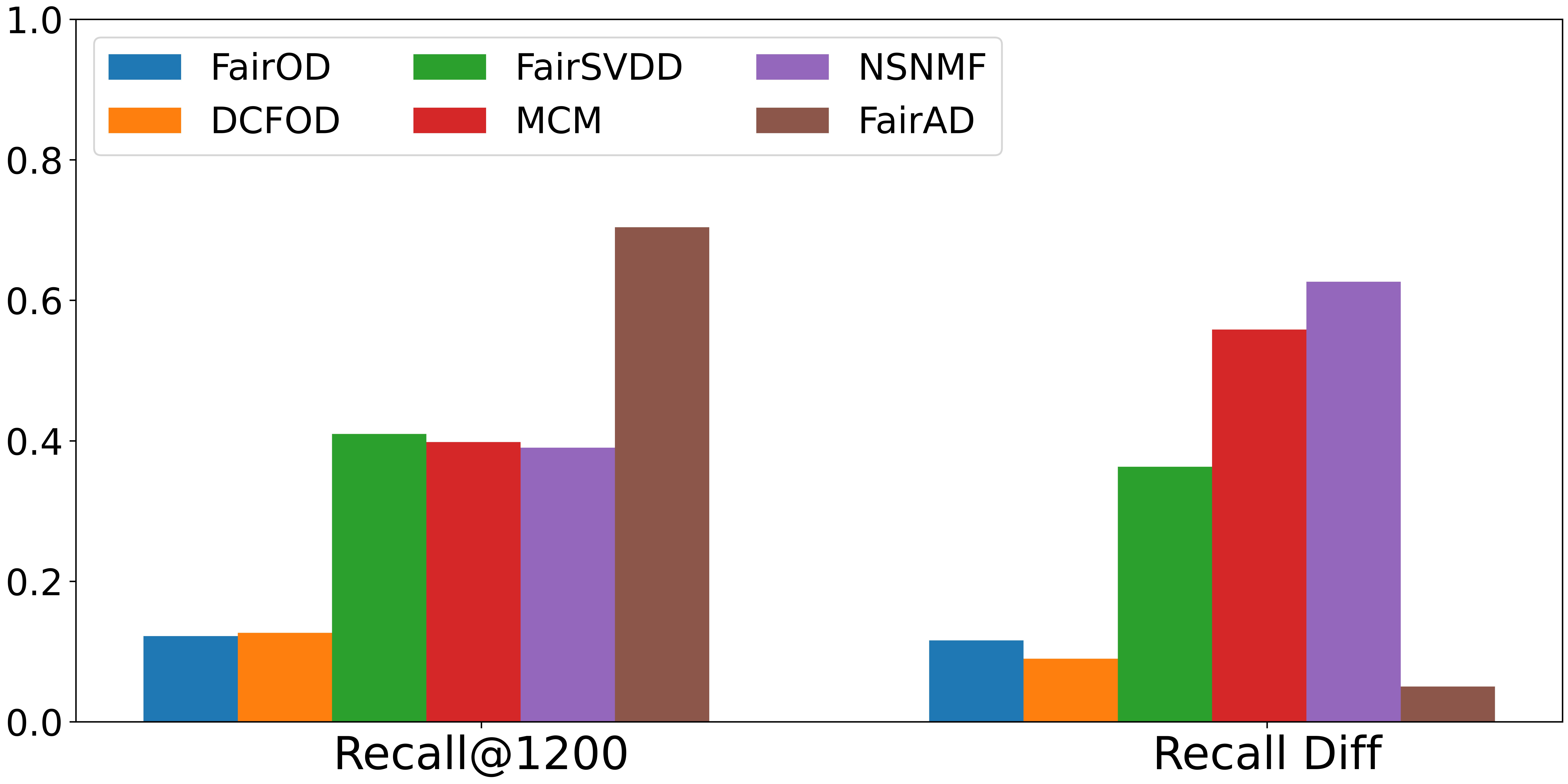}
    \captionof{figure}{Recall@1200 and absolute Recall difference of the existing methods on MNIST-USPS dataset.} 
    \label{fig:motivation}
  \end{minipage}\hfill
  \begin{minipage}[b]{.48\linewidth}
    \centering
    \scalebox{1.0}{
    \begin{tabular}{ *{3}{c} }
    \toprule
      Methods & Unprotected & Protected\\
      \midrule
      FairOD & 117(984) &	35(216) \\
      DCFOD & 124(970) &25(230)\\
      FairSVDD & 292(424) & 198(776)\\
      MCM & 238(327) &	234(873)\\
      NSNMF &196(294) &267(906)\\
      \name &630(809) &	247(391)\\
      \bottomrule
    \end{tabular}}
    \captionof{table}{True anomalies out of identified anomalies (number in the parentheses) of existing methods in each group on MNIST-USPS dataset. }
    \label{tab:motivation}
  \end{minipage}
\end{figure}
To tackle these challenges, in this paper, we propose \name, a fairness-aware contrastive learning-based anomaly detection method for the imbalanced scenario. \name\ mainly consists of two modules: 1) fairness-aware contrastive learning module; 2) re-balancing autoencoder module. Specifically, the fairness-aware contrastive learning module aims to maximize the similarity between the protected and unprotected groups to ensure fairness and address \textbf{C2}.
In addition, we encourage the uniformity of representations for examples within each group, as ensuring uniformity in contrastive learning can be beneficial for the imbalanced group scenario~\cite{DBLP:conf/nips/JiangCCW21}. To further address the negative impact of imbalanced data (i.e., \textbf{C1}), we propose the re-balancing autoencoder module utilizing the learnable weight to reweigh the importance of both the protected and unprotected groups. Combining the two modules, we design a simple yet efficient method \name\ with a theoretical guarantee of fairness. 
Our contributions are summarized below.
\begin{itemize}
    \vspace{-2.5mm}
    \item A fairness-aware anomaly detection method \name\ addressing the representation disparity and imbalanced data issues in the anomaly detection task. 
    \item Theoretical analysis showing that our proposed fair contrastive regularization term guarantees group fairness.
    \item The re-balancing autoencoder equipped with learnable weight alleviating the negative impact of the imbalanced group.
    \item Empirical studies demonstrating the effectiveness and efficiency of \name\ across multiple real-world datasets.
\end{itemize}

The rest of this paper is organized as follows. We first provide the preliminaries in Section \ref{sec:prelim} and then introduce our proposed fair anomaly detection method in Section \ref{sec:method}, followed by the theoretical fairness analysis in Section \ref{sec:theoretical_analysis}. Then, we systematically evaluate the effectiveness and efficiency of \name\ in Section \ref{sec:exp}. 
We finally conclude the paper in Section \ref{sec:conclusion}.

\DeclarePairedDelimiter{\norm}{\lVert}{\rVert}
\DeclarePairedDelimiter{\abs}{\lvert}{\rvert}

\newcommand{\lcon}{\mathcal{L}_{\mathrm{FAC}}}
\newcommand{\lfair}{\mathcal{L}_{\mathrm{fair}}}
\newcommand{\lunif}{\mathcal{L}_{\mathrm{unif}}}
\newcommand{\surcon}{\lcon'}
\newcommand{\surfair}{\lfair'}
\newcommand{\surunif}{\lunif'}

\section{Preliminaries}
\label{sec:prelim}
In this paper, we explore the fairness issue in the unsupervised anomaly detection task. Among the various fairness definitions proposed, there is no consensus about the best one to use. In this work, we focus on the group fairness notion which usually pursues the equity of certain metrics among the groups. For instance, Accuracy Parity \citep{zafar2017fairness} requires the same task accuracy between groups and Equal Opportunity \citep{hardt2016equality} requires the same true positive rate instead. Without loss of generality, we consider the groups here to be the protected group and the unprotected group (e.g., Black and Non-Black in race).
We are given a dataset $D=P\cup U$, where $P = \{x_i^P, y_i^P\}_{i=1}^n$ are examples from the protected group, $U = \{x_i^U, y_i^U\}_{i=1}^m$ from the unprotected group, and $x_i^P,x_i^U$ are sampled i.i.d from distributions $\pp,\pu$ over the input space $\mathbb{R}^d$ respectively. The ground-truth labels $y_i^P, y_i^U \in \mathcal{Y} = \{0, 1\}$ represent whether the example is an anomaly $(y=1)$ or not, which are given by deterministic labeling functions $a_P, a_U: \mathbb{R}^d \rightarrow \mathcal{Y}$, respectively. Note that we do not have access to the labels during training as we focus on the unsupervised anomaly detection setting.

The task of unsupervised anomaly detection is to find a hypothesis
$h: \mathbb{R}^d\rightarrow \mathcal{Y} $ which identifies a subset $\mathcal{A} \subset D $ whose elements deviate significantly from the majority of the data in $D$. This identification is done without the aid of labeled examples, meaning the algorithm must rely on the intrinsic properties of the data, such as distribution, density, or distance metrics, to discern between normal examples and anomalies. The risk of a hypothesis $h$ w.r.t. the true labeling function $a$ under distribution $\mathcal{D}$ using a loss function $\ell: \mathcal{Y} \times \mathcal{Y} \rightarrow \mathbb{R}_{+}$ is defined as: $R^\ell_\mathcal{D}(h, a) \coloneqq \mathbb{E}_{x \sim \mathcal{D}}\left[\ell(h(x), a(x))\right]$. We assume that $\ell$ satisfies the triangle inequality. For notation simplicity, we denote $R^\ell_P(h) \coloneqq R^\ell_{\pp}(h, a_P)$ and $R^\ell_U(h) \coloneqq R^\ell_{\pu}(h, a_U)$. The empirical risks over the protected group $P$ and the unprotected group $U$ are denoted by $\hat{R}^\ell_P$ and $\hat{R}^\ell_U$.

One direction of unsupervised AD is reconstruction-based autoencoder, such as USAD~\cite{DBLP:conf/kdd/AudibertMGMZ20} and DAAD~\cite{DBLP:conf/iccv/HouZZXPZ21}. Assuming the anomalies possess different features than the normal examples, given an autoencoder over the normal examples, it will be hard to compress and reconstruct the anomalies. The anomaly score can then be defined as the reconstruction loss for each test example. Formally, the autoencoder consists of two main components: an encoder $g_e : \mathbb{R}^d \rightarrow \mathbb{R}^r$ and a decoder $g_d : \mathbb{R}^r \rightarrow \mathbb{R}^d$, where $r$ is the dimensionality of the hidden representations. $g_e(x)$ encodes the input $x$ to a hidden representation $z$ that preserves the important aspects of the input. Then, $g_d(z)$ aims to recover $x'\approx x$, a reconstruction of the input from the hidden representation $z$. Overall, the autoencoder can be written as $G = g_d \circ g_e$, i.e. $G(x) = g_d (g_e (x))$. For a given autoencoder-based framework, the anomaly score for $x$ is computed using the reconstruction error as:
\begin{equation}
    s(x) = \norm{x - G(x)}^2,
\end{equation}
where all norms are $\ell_2$ unless otherwise specified. Anomalies tend to exhibit large reconstruction errors because they do not conform to the patterns in the data as coded by the autoencoder. This scoring function is generic in that it applies to many reconstruction-based AD models, which have different parameterizations of the reconstruction function $G$. 
Next, we will present our method design based on the autoencoder framework. For quick reference, we summarize the notation used in the paper in Table \ref{table:symbols} in Appendix.

\section{Proposed Method}
\label{sec:method}
Our proposed \name\ mainly consists of two modules: a Fairness-aware Contrastive Learning Module and a Re-balancing Autoencoder Module. 

\begin{figure}
    \centering
    \subcaptionbox{Failure of AD.\label{fig:sub_fail}}{\includegraphics[width=.24\linewidth]{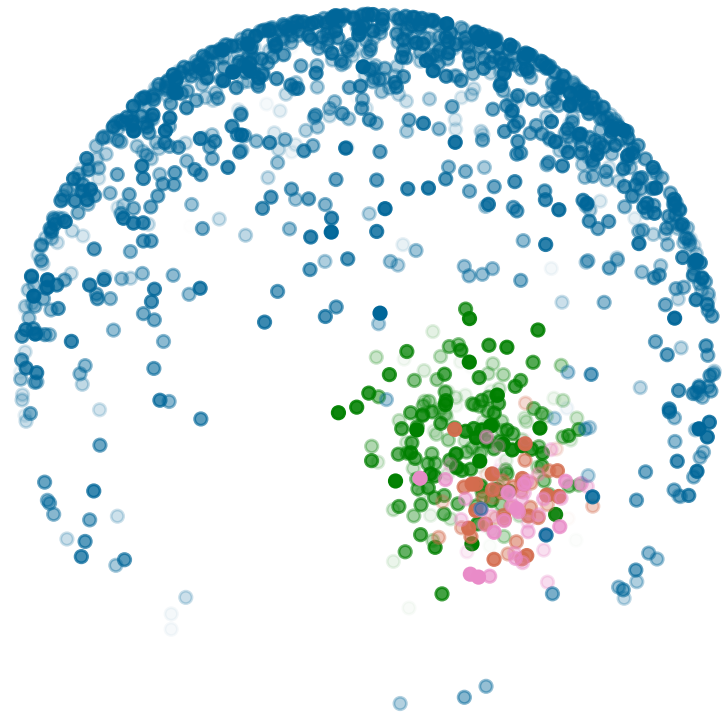}}
    \hfill 
    \subcaptionbox{Uniformity without fairness.\label{fig:over_unif}}{\includegraphics[width=.24\linewidth]{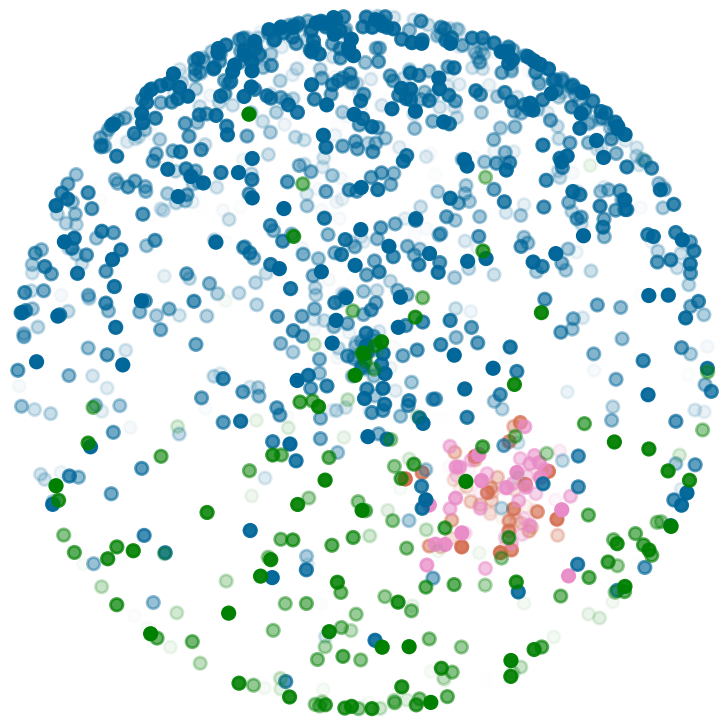}}
    \hfill 
    \subcaptionbox{\name: proper uniformity with fairness.\label{fig:proper_unif}}{\includegraphics[width=.24\linewidth]{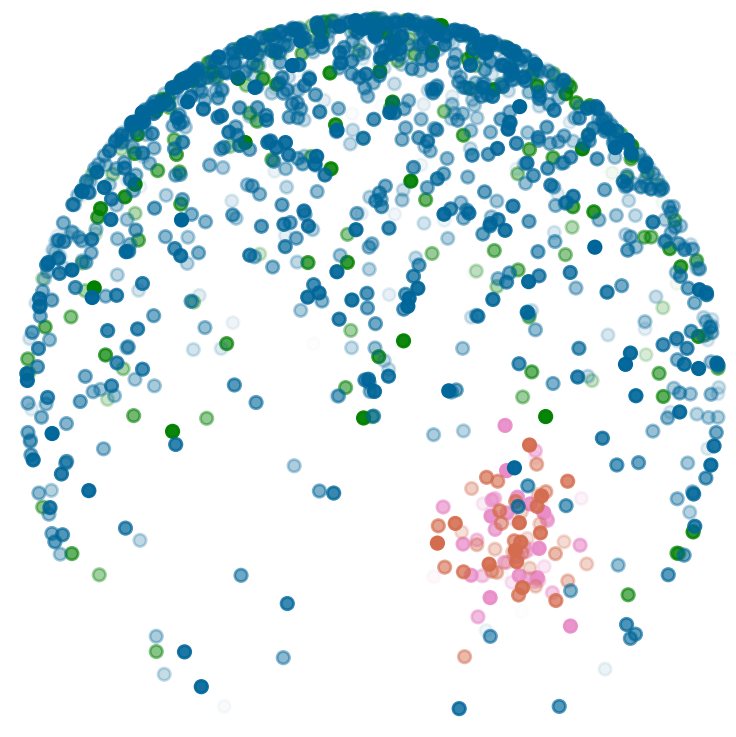}}

    \caption{Illustrations of uniformity. The blue dots and green dots denote the normal examples from the unprotected group and protected group respectively. The red and pink dots denote the anomalies from the unprotected group and protected group respectively.
    In (a), many existing AD methods overly flag the examples from the protected groups as anomalies. In (b), traditional contrastive regularization does not consider group fairness. In (c), our method ensures group fairness while maintaining proper uniformity.}
    \label{fig:contrastive}
    \vspace{-4mm}
\end{figure}


\subsection{Fairness-aware Contrastive Learning }
Existing anomaly detection models \cite{DBLP:conf/kdd/SongLL21, DBLP:conf/fat/ZhangD21, DBLP:conf/iccv/FioresiDS23} statistically focus on learning the frequent patterns (i.e., the unprotected group), while overlooking the under-represented patterns (i.e., the protected group) within the observed imbalanced data. Due to the lower contribution of protected groups to the overall learning objective (e.g., minimizing expected reconstruction loss), examples from the protected groups may experience systematically higher errors. Thus, they tend to erroneously label most normal examples from the protected group as anomalies, producing unfair outcomes as shown in Figure \ref{fig:sub_fail}. 

Recent works~\cite{DBLP:conf/icml/0001I20, DBLP:conf/iclr/SohnLYJP21} have shown that encouraging uniformity with contrastive learning can alleviate this issue by pushing examples uniformly distributed in the unit hypersphere, as illustrated in Figure \ref{fig:over_unif}. Therefore, one naive solution is to implement contrastive learning \cite{chen2020simple} to learn representations by distinguishing different views of one example from other examples 
as follows: 
\begin{small}
\begin{align}
    \mathcal{L}_\mathrm{SimCLR} = - \sum_{z_j\in P \cup U} \log \frac{\text{sim}(z_j, z_j^+)}{\sum_{z_k\in P \cup U} \text{sim}(z_j, z_k)},
\end{align}
\end{small}
where $z_j= g_e(x_j)$ is the hidden representation, $U, P$ are slightly abused to denote the empirical distributions of the hidden representations of the unprotected and protected group, $z_j^+$ is obtained by an augmentation function to form a positive pair with $z_j$, and $\text{sim}(a, b)= \exp(\frac{a^Tb}{|a||b|})$. By minimizing $\mathcal{L}_\mathrm{SimCLR}$, we encourage the uniformity of the representations of the two groups.

However, as shown in Figure \ref{fig:over_unif}, although the protected examples deviate from anomalies after encouraging uniformity, group fairness could not be guaranteed by the traditional contrastive learning loss. 
To promote fairness between the protected group and the unprotected group, we further propose to maximize the cosine similarity between the representations of the two groups, as shown in Figure \ref{fig:proper_unif}. Formally, we minimize the following fairness-aware contrastive loss:
\begin{align}
\label{equ:contrastive}
\lcon =& - \log \frac{\frac{1}{mn}\sum_{j\in [n]}\sum_{k\in [m]}\text{sim}\left(z_j^P, z_k^U\right)}{\frac{1}{m(m-1)}\sum_{j\ne k} \text{sim}\left(z_j^U, z_k^U\right) +  \frac{1}{n(n-1)}\sum_{j\ne k} \text{sim}\left(z_j^P, z_k^P\right)} \\
=&  \underbrace{-\log\left(\frac{\sum_{j}\sum_{k}\text{sim}\left(z_j^P, z_k^U\right)}{mn}\right)}_{\lfair} \nonumber 
+ \underbrace{\log\left(\frac{\sum_{j\ne k} \text{sim}\left(z_j^U, z_k^U\right)}{m(m-1)} + \frac{\sum_{j\ne k} \text{sim}\left(z_j^P, z_k^P\right)}{n(n-1)}\right)}_{\lunif}  \nonumber
\end{align}

Following the interpretation of contrastive loss in~\cite{DBLP:conf/icml/0001I20}, the numerator (i.e., $\lfair$) can be interpreted as ensuring the fairness of two groups and the denominator (i.e., $\lunif$) can be interpreted as encouraging the diversity or uniformity of the representations in the unit hypersphere.
Besides, we show that our proposed fair contrastive regularization term guarantees group fairness with theoretical support in Section~\ref{sec:theoretical_analysis}.

\subsection{Re-balancing Autoencoder}
We then introduce the autoencoder-based module of our method. The existing autoencoder-based AD frameworks~\cite{DBLP:conf/kdd/SongLL21, DBLP:conf/kdd/AudibertMGMZ20} aim to optimize the following reconstruction loss:
\begin{equation}
\label{equ:recons}
    \mathcal{L}_\mathrm{REC} = \sum_{x_i\in {P \cup U}} \norm{x_i - G(x_i)}^2 =  \underbrace{\sum_{i=1}^{n} \norm{x_i^P - G\left(x_i^P\right)}^2}_{\mathcal{L}_P} + \underbrace{\sum_{i=1}^{m} \norm{x_i^U - G\left(x_i^U\right)}^2}_{\mathcal{L}_U}.
\end{equation}

As these AD approaches fail to consider the data imbalance nature of the protected and unprotected groups, the learning objective in Equation (\ref{equ:recons}) tends to solely concentrate on learning frequent patterns of the unprotected group (i.e., $\mathcal{L}_U$), yielding higher reconstruction errors for the examples from the protected group. 
Consequently, existing methods usually overly flag the examples from the protected group as anomalies, thus having a higher recall difference, as illustrated in Figure \ref{fig:motivation}. 

To address the data imbalance issue between the two groups (i.e., \textbf{C1} in the introduction), we design a re-balancing autoencoder by minimizing the reweighted reconstruction loss as follows:
\begin{equation}
\mathcal{L}_{\mathrm{REC}}  = (1 - \epsilon)  \mathcal{L}_U + \epsilon  \mathcal{L}_P ,
\end{equation}
A proper weight $\epsilon$ should promote the model fitting on the normal examples in both protected and unprotected groups. Consider the four subgroups of data samples in the task of fair anomaly detection: unprotected/protected normal examples (UN/PN) and unprotected/protected anomalies (UA/PA). Since ideally the model should only fit UN and PN, we assume that the model is capable of fitting two out of the four subgroups. For the design of $\epsilon$ we have the following lemma:

\begin{lemma}\label{lemma:epsilon} 
Let $\mathcal{L}_0^t$ denote the loss of the unfitted model on the subgroup $t \in$ \{UN, PN, UA, PA\}, and let $\mathcal{L}_1^t$ denote the loss of the fitted model on the subgroup $t$, and $\Delta^t = \mathcal{L}_0^t - \mathcal{L}_1^t > 0$ means the difference of loss between the fitted model and the unfitted one on the subgroup $t$. A proper weight that promotes model fitting on normal examples in both protected and unprotected groups should be within the range $ \frac{\Delta^{UA}}{\Delta^{UA}+ \Delta^{PN}} < \epsilon < \frac{\Delta^{UN}}{\Delta^{UN}+ \Delta^{PA}}$.

\end{lemma}

Although $\frac{\Delta^{UA}}{\Delta^{UA}+ \Delta^{PN}}$ and $\frac{\Delta^{UN}}{\Delta^{UN}+ \Delta^{PA}}$ are unknown, we propose a design of $\epsilon$ that provably lies in this range: $\epsilon = \frac{\mathcal{L}_0^U - \mathcal{L}_U}{\mathcal{L}_0^U - \mathcal{L}_U + \mathcal{L}_0^P - \mathcal{L}_P} $ where $\mathcal{L}^U_0 = \mathcal{L}^{UN}_0 +  \mathcal{L}^{UA}_0$ and $\mathcal{L}^P_0 = \mathcal{L}^{PN}_0 +  \mathcal{L}^{PA}_0$.
 We estimate $\mathcal L^U_0 = \sum_{i \in U} \|x_i - \overline{G(x)}\|^2$
 where $\overline{G(x)} = \frac{1}{|U|} \sum_{i \in U} G(x_i) $,  and $\mathcal L^P_0 = \sum_{i \in P} \|x_i - \overline{G(x)}\|^2$ where $\overline{G(x)} = \frac{1}{|P|} \sum_{i \in P} G(x_i) $.
 The proof of Lemma \ref{lemma:epsilon} and the justification of our design are provided in Appendix \ref{app:prooflemmaepsilon}. 
Finally, the overall training scheme of \name\ is to minimize:
\begin{equation*}
\mathcal{L}_{\mathrm{overall}} = \mathcal{L}_{\mathrm{REC}} + \alpha  \mathcal{L}_{\mathrm{FAC}},
\end{equation*}
where $\alpha$ is a hyperparameter to balance the reconstruction loss and the contrastive loss. During the inference stage, we rank the reconstruction error of each example and pick the top $k$ examples as anomalies.

\section{Theoretical Analysis}
\label{sec:theoretical_analysis}
In this section, we show how our proposed method promotes fairness. We focus on the group fairness notions where the difference in certain performance metrics between the two groups is considered. We first introduce the definition of $f$-divergence to help formulate an upper bound on the performance difference of \name:
\begin{definition}($f$-divergence \cite{ali1966general} )
    Let \( P \) and \( Q \) be two distribution functions with densities \( p \) and \( q \), respectively. 
Let \( p \) be absolutely continuous w.r.t \( q \) and both be absolutely continuous with respect to a base measure \( dx \). 
Let \( f: \mathbb{R}_+ \rightarrow \mathbb{R} \) be a convex, lower semi-continuous function that satisfies \( f(1) = 0 \). 
The \( f \)-divergence \( D_{f} \) is defined as:
\begin{equation}
D_{f} (P \parallel Q) = \int q(x) f \left( \frac{p(x)}{q(x)} \right) dx.
\end{equation}
\end{definition}

Many popular divergences that are heavily used in machine learning are special cases of $f$-divergences, and we include some in Table \ref{tab:fdivergences} in Appendix \ref{appendix:divergence}. \cite{nguyen2010estimating} derived a general variational approach for estimating $f$-divergence from examples by transforming the estimation problem into a variational optimization problem. They show that any $f$-divergence can be written as:
\begin{equation}
\label{equ:variation}
D_{f}(P \parallel Q) \geq \sup_{T \in \mathcal{T}} \mathbb{E}_{x \sim P}[T(x)] - \mathbb{E}_{x \sim Q}[f^*(T(x))]
\end{equation}
where \( f^* \) is the (Fenchel) conjugate function of \( f : \mathbb{R}_+ \rightarrow \mathbb{R} \) defined as \( f^*(y) := \sup_{x \in \mathbb{R}_+}\{xy - f(x)\} \), \( T : \mathcal{X} \rightarrow \text{dom} \, f^* \), and \( \mathcal{T} \) is the set of all measurable functions. 

Given that $D_{f}(P\parallel Q)$ involves the supremum over all measurable functions and does not account for the hypothesis class, and that it cannot be estimated from finite examples of arbitrary distributions \cite{kifer2004detecting}, we further consider a discrepancy which helps relieve these issues based on the variational characterization of $f$-divergence in Equation (\ref{equ:variation}):

\begin{definition} ($D^f_{h, \mathcal{H}}$ discrepancy \cite{acuna2021f})
    Let \( f^* \) be the Fenchel conjugate of a convex, lower semi-continuous function \( f \) that satisfies \( f(1) = 0 \), and let \( \hat{T} \) be a set of measurable functions such that \( \hat{T} = \{\ell(h(x), h'(x)) : h, h' \in \mathcal{H}\} \). 
We define the discrepancy between the two distributions \( P \) and \( Q \) as:

\begin{small}
\begin{equation*}
\begin{aligned}
D^f_{h, \mathcal{H}} (P \parallel Q) := & \sup_{h' \in \mathcal{H}} | \mathbb{E}_{x \sim P} [\ell(h(x), h'(x))] - \mathbb{E}_{x \sim Q} [f^* (\ell(h(x), h'(x)))] |
\end{aligned}
\end{equation*}    
\end{small}
\end{definition}

From the definition we can easily get $D^f_{h, \mathcal{H}} (P \parallel Q) \leq D_f (P \parallel Q)$.
We then introduce a useful tool, Rademacher complexity\cite{shalev2014understanding} (detailed definition provided in Appendix \ref{appendix:rademacher}), and its commonly used property:
\begin{lemma}(Property of Rademacher complexity \cite{mohri2018foundations}). For any \(\delta \in (0, 1)\), with probability at least $1- \delta$ over the draw of an i.i.d. sample $D$ of size $|D|$, the following inequality holds for all $h \in \mathcal{H}$:
\begin{equation}
\label{equ:rademacher}
|R^\ell_D(h) - \hat{R}^\ell_D(h)| \leq 2\mathfrak{R}_D(\ell \circ \mathcal{H}) + \sqrt{\frac{\log \frac{1}{\delta}}{2|D|}}
\end{equation}

\end{lemma}
where $\mathfrak{R}_D(\ell \circ \mathcal{H})$ is the  Rademacher complexity of the function class $\ell \circ \mathcal{H}$ given data $D$.
With this property, we now show that $D^f_{h,\mathcal{H}}$ can be estimated from finite examples:
\begin{lemma} 
\label{lemma:df}
Suppose \(\ell : \mathcal{Y} \times \mathcal{Y} \rightarrow [0, 1]\), \(f^*\) is L-Lipschitz continuous, and \([0, 1] \subseteq \text{dom} \, f^*\). Let \(U\) and \(P\) be two empirical distributions corresponding to datasets containing \(m\) and $n$ data points sampled i.i.d. from \(P_U\) and \(P_P\), respectively. Let us note \(\mathfrak{R}\) the Rademacher complexity of a given hypothesis class, and \(\ell \circ \mathcal{H} := \{x \mapsto \ell(h(x), h'(x)) : h, h' \in \mathcal{H}\}\). For any \(\delta \in (0, 1)\), with probability at least \(1 - \delta\), we have:
\end{lemma}
\vspace{-5mm}
\begin{small}
\begin{equation}
\begin{aligned}
 |D^f_{h,\mathcal{H}}(P_U \| P_P) - D^f_{h,\mathcal{H}}(U \| P)| \leq 2\mathfrak{R}_{P_U}(\ell \circ \mathcal{H}) + 2 L \mathfrak{R}_{P_P}(\ell \circ \mathcal{H}) + \sqrt{\frac{\log \frac{1}{\delta}}{2n}} + \sqrt{\frac{\log \frac{1}{\delta}}{2m}}
\end{aligned}
\end{equation}
\end{small}

Lemma \ref{lemma:df} shows that the empirical $D^f_{h,\mathcal{H}}$ converges to the true discrepancy, and the gap is bounded by the complexity of the hypothesis class and the number of examples.

\subsection{Fairness Bounds}
We now provide a fairness bound to estimate the performance difference between the protected and unprotected groups using the previously defined $D^f_{h, \mathcal{H}}$ divergence.

\begin{theorem}
\label{theorem:riskdiff}
Let $h^*$ be the ideal joint hypothesis, i.e., $h^* = \arg \min_{h \in \mathcal{H}} R^\ell_U(h) + R^\ell_P(h)$. The risk difference between the two groups is upper bounded by:
\begin{equation}
 R^\ell_P(h) - R^\ell_U(h) \leq  D^f_{h,\mathcal{H}}(P_U\|P_P) + R^\ell_U(h^*) + R^\ell_P(h^*).
\end{equation}
\end{theorem}

For the upper bound on the RHS, the first term corresponds to the discrepancy between the marginal distributions, and the remaining two terms measure the ideal joint hypothesis.  If $\mathcal{H}$ is expressive enough and the labeling functions are similar, the last two terms could be reduced to a small value. 

\begin{theorem}
\label{theorem:bound}
(Fairness with Rademacher Complexity) Under the same conditions as in Lemma \ref{lemma:df}, for any \(\delta \in (0, 1)\), with probability at least \(1 - \delta\), we have:
\end{theorem}
\vspace{-5mm}
\begin{small}
\begin{equation}
\begin{aligned}
 & R^\ell_P(h) -  R^\ell_U(h)\leq  D_f(U \| P) + \hat{R}^\ell_U(h^*) + \hat{R}^\ell_P(h^*) \\
& + 4\mathfrak{R}_U(\ell \circ \mathcal{H})+ 2 (L+1) \mathfrak{R}_P(\ell \circ \mathcal{H}) + 2 \sqrt{\frac{\log \frac{1}{\delta}}{2m}}  + 2 \sqrt{\frac{\log \frac{1}{\delta}}{2n}}
 \end{aligned}
\end{equation}
\end{small}

\begin{table*}[t]
    \centering
    \caption{Characteristics of datasets.}
    \vspace{-1mm}
    \scalebox{0.75}{
    \begin{tabular}{c|c|c|c|c|c|c|c}
    \toprule
    \multirow{2}{*}{Datasets} & \multicolumn{2}{c|}{Unprotected Group} & \multicolumn{2}{c|}{Protected Group} & \multirow{2}{*}{\#Features } & \multirow{2}{*}{Sensitive Attribute } & \multirow{2}{*}{Anomaly Definition}\\ \cline{2-5}
           & \#Instances &\#Anomaly  & \#Instances &\#Anomaly  &    &  &  \\ \midrule 
    MNIST-USPS      & 7,785       & 882       & 1,876       & 323       & 1,024         &    Source of the digits           &  Digit 0 or not \\ \midrule
    MNIST-Invert     & 7,344       & 441       & 408         & 38        & 1,024         &    Color of the digits   &   Digit 0 or not \\ \midrule
    COMPAS          & 1,839       & 325       & 299         & 39        & 8             &    Race                    &  Reoffending or not  \\ \midrule
    CelebA    & 41,919      & 4,008     & 7,300       & 1,142     & 39            &    Gender                  & Attractive or not\\ \bottomrule
    \end{tabular}}
    \label{tab:dataset_statistics}
\end{table*}

Under the assumption of an ideal joint hypothesis, fairness can be improved by minimizing the discrepancy between the two distributions and regularizing the model to limit the complexity of the hypothesis class. The detailed proofs of the lemma and the theorems are in Appendix \ref{sec:allproofs} and \ref{app:proof45}. We further motivate why minimizing the objective $\lcon$ leads to small $D_f(U || P)$
for total variation in Appendix \ref{appendix:fairnessbound}.

\section{Experiments}
\label{sec:exp}
In this section, we experimentally analyze and compare our proposed \name\ with other anomaly detection methods. We try to answer the following research questions:
\begin{desclist}
    \item RQ1: How does \name\ compare with other baselines on imbalanced datasets?
    \item RQ2: How does \name\ perform with different ratios of the two groups?
    \item RQ3: How does each module contribute to \name?
\end{desclist}






\subsection{Experimental Setup}
\noindent \textbf{Datasets:} We conduct experiments on two image datasets, MNIST-USPS and MNIST-Invert \cite{DBLP:conf/fat/ZhangD21}, and two tabular datasets, COMPAS \cite{angwin2022machine} and CelebA \cite{liu2015deep}. The characteristics of the datasets are presented in Table \ref{tab:dataset_statistics}. 

\noindent \textbf{Baseline Methods:}
In our experiments, we compare our proposed framework \name\ with the following fairness-aware anomaly detection baselines:
 (1) \textbf{FairOD}~\cite{shekhar2021fairod}, a fair AD method which incorporates the prescribed criteria into its training;
 (2) \textbf{DCFOD}~\cite{DBLP:conf/kdd/SongLL21}, a fair deep clustering-based method, which leverages deep clustering to discover the intrinsic cluster structure and out-of-structure instances;
(3) \textbf{FairSVDD}~\cite{DBLP:conf/fat/ZhangD21}, an adversarial network to de-correlate the relationships between sensitive attributes and the learned representations. 
 We also compare with the following fairness-agnostic AD baselines:
 (4) \textbf{MCM}~\cite{anonymous2024mcm}, a masked modeling method to address AD by capturing intrinsic correlations between features in the training set;
 (5) \textbf{NSNMF}~\cite{ahmed2021neighborhood}, a non-negative matrix factorization method, which incorporates the neighborhood structural similarity information to improve the anomaly detection performance; 
 (6) \textbf{ReContrast}~\cite{guo2023recontrast}, a reconstructive contrastive learning-based method for domain-specific anomaly detection. Notice that as ReContrast is designed for image data, we only evaluate it on MNIST-USPS and MNIST-Invert datasets.

\noindent \textbf{Metrics:} To measure the model performance and group fairness, we choose three widely-used metrics~\cite{shekhar2021fairod, DBLP:conf/fat/ZhangD21, ahmed2021neighborhood}: (1) \textbf{Recall@k}, which measures the proportion of anomalies found in the top-k recommendations; (2) \textbf{ROCAUC}, which computes the area under the receiver operating characteristic curve; 
(3) \textbf{Rec Diff}, which measures the absolute value of the recall difference between two groups.

\noindent \textbf{Training details:} For the COMPAS dataset, we use a two-layer MLP with hidden units of [32, 32]. For all the other datasets, we use MLP with one hidden layer of dimension 128.  All our experiments were executed using one Tesla V100 SXM2 GPUs, supported by a 12-core CPU operating at 2.2GHz. We provide our implementation in Appendix \ref{app:code}.
\begin{table*}[t]
    \centering
    \caption{Performance on Image Datasets. The best score is marked in bold.}
    \vspace{-3mm}
    \scalebox{0.7}{
    \begin{tabular}{c|cccc|cccc}
        \toprule
        \multirow{2}{*}{Methods} & \multicolumn{4}{c|}{MNIST-USPS (K=1200)} & \multicolumn{4}{c}{MNIST-Invert (K=500)}\\\cmidrule{2-9}
         & Recall@K & ROCAUC & Rec Diff  &Time(s) & Recall@K & ROCAUC & Rec Diff  &Time(s) \\ \midrule
         FairOD     & \stat{12.20}{1.33} & \stat{49.96}{0.34}  & \stat{11.59}{0.78}    & 28.45    & \stat{7.66}{0.84} & \stat{50.46}{0.20} & \stat{8.02}{1.43}   & 20.33\\ \midrule
         DCFOD      & \stat{12.67}{0.39} & \stat{50.18}{0.30} & \stat{8.99}{1.02}     & 698.27   & \stat{6.89}{1.11} & \stat{50.51}{0.66} & \stat{7.24}{2.85}   & 1287.52\\ \midrule
         FairSVDD   & \stat{15.46}{1.67} & \stat{58.28}{1.22} & \stat{13.73}{2.64}& 766.36& \stat{12.45}{0.88}& \stat{49.69}{4.43}&\stat{12.43}{2.16}& 846.45\\  \midrule
         MCM        & \stat{39.83}{0.20} & \stat{78.84}{1.07}& \stat{55.83}{0.83}& 416.11& \stat{25.33}{0.60}&\stat{80.95}{0.64}&\stat{80.15}{1.45}& 750.17\\  \midrule
         NSNMF      & \stat{39.03}{0.99} & \stat{65.20}{0.56} & \stat{62.64}{4.66} & 28.2 & \stat{51.77}{0.75} & \stat{74.16}{0.40}  & \stat{51.43}{2.01} & 18.95 \\  \midrule
         Recontrast   & \stat{64.80}{3.69} & \stat{83.91}{4.49} & \stat{40.77}{6.83} & 118.99 & \stat{64.45}{1.88} & \stat{86.11}{5.89}  & \stat{55.33}{13.45} & 119.54 \\ \midrule
         \name       & 
         \bfstat{67.16}{0.37} & \bfstat{91.27}{0.49} & \bfstat{3.73}{2.13} & 122.84 & \bfstat{72.37}{0.32}& \bfstat{98.03}{0.01}& \bfstat{6.75}{0.34}& 52.28\\
        \bottomrule
    \end{tabular}}
    \label{tab:image_datasets}
\end{table*}

\begin{table*}[t]
    \centering
    \caption{Performance on Tabular Datasets. The best score is marked in bold.}
    \vspace{-3mm}
    \scalebox{0.7}{
    \begin{tabular}{c|cccc|cccc}
        \toprule
        \multirow{2}{*}{Methods} & \multicolumn{4}{c|}{COMPAS (K=350)} & \multicolumn{4}{c}{CelebA (K=5000)}\\\cmidrule{2-9}
         & Recall@K & ROCAUC & Rec Diff  &Time(s) & Recall@K & ROCAUC & Rec Diff  &Time(s) \\ \midrule
         FairOD      & \stat{16.58}{2.60} & \stat{50.12}{1.57}  & \stat{8.29}{1.28} & 4.20   & \stat{8.91}{0.16} & \stat{49.92}{0.14} &  \bfstat{0.67}{0.68} & 78.87 \\ \midrule
         DCFOD       & \stat{15.94}{2.35} & \stat{49.74}{1.42}  & \stat{9.75}{2.15} & 116.82 & \stat{9.75}{0.82} & \stat{49.93}{0.17} & \stat{7.51}{1.33} & 2523.57 \\ \midrule
         FairSVDD    & \stat{15.29}{2.25} & \stat{52.60}{5.48} &\stat{11.59}{4.28} &6.62& \stat{10.16}{0.58}& \stat{58.40}{1.21} &\stat{10.86}{2.02}& 248.95\\  \midrule
         MCM         & \stat{20.97}{0.69} & \stat{50.56}{0.51} &\stat{6.26}{2.90}& 38.23& \stat{11.07}{0.50}& \stat{46.08}{3.95}& \stat{27.06}{11.99}& 632.81\\  \midrule
         NSNMF   & \stat{22.89}{0.27} & \stat{57.96}{0.81} & \stat{36.04}{0.67} & 7.54 & \stat{10.88}{0.66} & \stat{50.40}{0.37}  & \stat{8.05}{1.63} & 1870.06\\  \midrule
         \name       & \bfstat{34.43}{0.42}& \bfstat{61.85}{0.52}& \bfstat{5.81}{4.36}& 17.94 &\bfstat{11.94}{0.67}& \bfstat{59.41}{0.58}& \stat{4.66}{1.72}& 52.81\\
        \bottomrule
    \end{tabular}}
    
    \label{tab:tabular_datasets}
\end{table*}

\subsection{Effectiveness and Efficiency of \name\ (RQ1)}
\label{effectiveness}
We first evaluate the effectiveness and efficiency of \name\ through comparison with baselines across four datasets by three independent runs.  The task performance (\ie, Recall@$K$ and ROCAUC), group fairness measure (\ie, Rec Diff), and their average training time are presented in Tables \ref{tab:image_datasets} and \ref{tab:tabular_datasets} (See Appendix \ref{more_recall_k} for Recall@$K$ with different $K$). We can observe that the fair AD baselines (FairOD, DCFOD, and FairSVDD) typically exhibit low discrepancies in recall. However, they also tend to suffer from reduced recall rates and ROCAUC scores, suggesting a compromise in overall task performance to enhance fairness.
On the other hand, the baselines that do not account for fairness, including MCM, NSNMF, and ReContrast, demonstrate high recall rates and ROCAUC scores but often at the expense of fairness, as evidenced by significant disparities across groups (\ie, a higher Rec Diff).  
Our \name\ instead addresses the challenge of imbalance between the groups and the imbalanced distributions of normal examples and anomalies. 
Remarkably, \name\ not only excels in task performance but also elevates the level of fairness, underscoring the effectiveness of our design in harmonizing fairness with anomaly detection in scenarios characterized by data imbalance. On the other hand, the training time of \name\ is always among the top 4 fastest methods across different datasets, showing the efficiency of our method.


\vspace{-2mm}
\begin{table*}[b]
    \centering
    \caption{Performance on MNIST-USPS with different ratios. The best score is marked in bold.}
    \vspace{-1mm}
    \scalebox{0.65}{
    \begin{tabular}{c|ccc|ccc|ccc}
        \toprule
        \multirow{2}{*}{Methods} & \multicolumn{3}{c|}{$|U|:|P|=1:1$ (K=650)} & \multicolumn{3}{c|}{$|U|:|P|=2:1$ (K=1000)} & \multicolumn{3}{c}{$|U|:|P|=4:1$ (K=1200)}\\\cmidrule{2-10}
         & Recall@K & ROCAUC & Rec Diff & Recall@K & ROCAUC & Rec Diff & Recall@K & ROCAUC & Rec Diff  \\ \midrule
         FairOD     & \stat{17.52}{1.17} & \stat{50.13}{0.64} & \stat{2.14}{0.62} & \stat{17.30}{1.24} & \stat{49.73}{0.74} & \stat{5.11}{0.55} & \stat{13.61}{0.22} & \stat{50.22}{0.13} & \stat{10.58}{1.01} \\ \midrule
         DCFOD      & \stat{17.08}{0.50} & \stat{50.09}{0.30} & \stat{3.25}{0.94} & \stat{16.92}{0.81} & \stat{49.54}{0.42} & \stat{2.76}{0.51} & \stat{14.14}{1.03} & \stat{50.44}{0.60} & \stat{7.11}{0.83} \\ \midrule
         FairSVDD   & \stat{24.56}{2.95} & \stat{54.87}{3.36} & \stat{14.24}{7.90} & \stat{18.09}{3.46} & \stat{52.77}{1.72} & \stat{4.85}{3.75} & \stat{21.10}{2.79} & \stat{63.46}{9.56} & \stat{18.38}{4.91} \\ \midrule
         MCM        & \stat{52.22}{1.35} & \stat{74.62}{1.24} & \stat{17.13}{2.73} & \stat{53.63}{1.76} & \stat{76.80}{1.04} & \stat{8.17}{6.36} & \stat{41.99}{4.06} & \stat{74.09}{0.45} & \stat{22.85}{4.60} \\ \midrule
         NSNMF       & \stat{48.71}{0.39} & \stat{68.96}{0.24} & \stat{40.25}{2.17} & \stat{41.07}{2.77} & \stat{64.08}{1.67} & \stat{54.18}{3.11} & \stat{38.87}{1.09} & \stat{64.71}{0.63} & \stat{62.98}{1.47} \\ \midrule
         Recontrast       & \stat{45.92}{1.85} & \stat{80.17}{3.08} & \stat{42.52}{3.31} & \stat{51.39}{1.75} & \stat{83.13}{2.94} & \stat{26.16}{1.79} & \stat{57.69}{2.36} & \stat{79.17}{4.09} & \stat{20.69}{3.57} \\ \midrule
         \name      & \bfstat{65.58}{0.47} &\bfstat{85.38}{0.37}& \bfstat{0.93}{0.87}& \bfstat{66.84}{0.83}& \bfstat{89.17}{0.09}&\bfstat{2.32}{1.08}& \bfstat{66.63}{0.72}&\bfstat{90.15}{0.22}& \bfstat{1.84}{0.68} \\
         \bottomrule
    \end{tabular}}
    \label{tab:ratio_mnist}
\end{table*}

\begin{table*}[t]
    \centering
     \caption{Performance on COMPAS dataset with different ratios. The best score is marked in bold.}
     \vspace{-1mm}
    \scalebox{0.65}{
    \begin{tabular}{c|ccc|ccc|ccc}
        \toprule
        \multirow{2}{*}{Methods} & \multicolumn{3}{c|}{$|U|:|P|=1:1$ (K=80)} & \multicolumn{3}{c|}{$|U|:|P|=2:1$ (K=120)} & \multicolumn{3}{c}{$|U|:|P|=5:1$ (K=240)}\\\cmidrule{2-10}
         & Recall@K & ROCAUC & Rec Diff & Recall@K & ROCAUC & Rec Diff & Recall@K & ROCAUC & Rec Diff \\ \midrule
         FairOD     & \stat{13.68}{2.67} & \stat{50.10}{0.85} & \stat{11.97}{1.48} & \stat{13.11}{0.50} & \stat{50.11}{0.74} & \stat{6.60}{0.97} & \stat{12.54}{1.37} & \stat{49.58}{0.87} & \stat{7.68}{0.72} \\ \midrule
         DCFOD      & \stat{11.54}{4.62} & \stat{48.50}{2.69} & \stat{7.69}{4.445} & \stat{15.95}{3.00} & \stat{53.28}{0.75} & \stat{10.68}{2.67} & \stat{12.96}{2.02} & \stat{49.76}{1.16} & \stat{6.36}{0.70} \\ \midrule
         FairSVDD   & \stat{16.24}{2.18} & \stat{52.34}{1.38} & \stat{6.84}{3.20} & \stat{14.53}{1.84} & \stat{51.69}{2.15} & \stat{7.69}{3.77}  & \stat{14.10}{4.53} & \stat{50.04}{4.98} & \stat{14.87}{7.54}\\ \midrule
         MCM        & \stat{18.38}{0.60} & \stat{40.77}{0.25} & \stat{7.69}{3.63} & \stat{16.24}{0.01} & \stat{40.42}{0.12} & \stat{10.26}{4.80} & \stat{18.81}{0.60} & \stat{44.04}{0.15} & \stat{5.76}{2.31} \\ \midrule
         NSNMF       & \stat{20.08}{0.74} & \stat{53.86}{0.42} & \stat{14.53}{10.36} & \stat{19.09}{1.31} & \stat{53.28}{0.75} & \stat{10.68}{2.67}  & \stat{20.09}{2.22} & \stat{53.86}{1.28} & \stat{10.77}{5.40} \\ \midrule
         \name      & \bfstat{29.91}{0.74} & \bfstat{61.87}{1.89}& \bfstat{3.42}{1.48}& \bfstat{28.42}{0.43}& \bfstat{57.39}{2.84}& \bfstat{1.92}{1.72}& \bfstat{29.77}{1.31}& \bfstat{58.05}{1.34}& \bfstat{4.83}{0.78} \\
         \bottomrule
    \end{tabular}}
    \label{tab:ratio_compas}
\end{table*}

\subsection{Data Imbalance Study (RQ2)}
To further study the performance of \name\ in handling imbalanced data, we vary the levels of group imbalance within the image dataset MNIST-USPS and the tabular dataset COMPAS.  We report the average results of three independent runs in Table \ref{tab:ratio_mnist} and Table \ref{tab:ratio_compas}. The tables demonstrate that \name\ consistently outperforms the baselines in terms of both task efficacy and fairness across different group ratios. The advantages of using \name\ become more pronounced with increasing level of group imbalance. For instance, while the performance of fair AD baselines drops with higher imbalance ratios on the MNIST-USPS dataset, \name\ adeptly sustains superior task performance alongside enhanced fairness levels, showcasing its robustness against data imbalance.

\begin{figure}[t]
    \centering
    \begin{tabular}{cc}
         \hspace{-3mm}\includegraphics[width=0.30\linewidth]{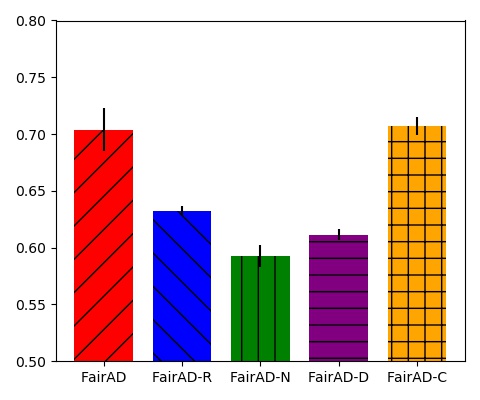} & \hspace{-3mm}\includegraphics[width=0.30\linewidth]{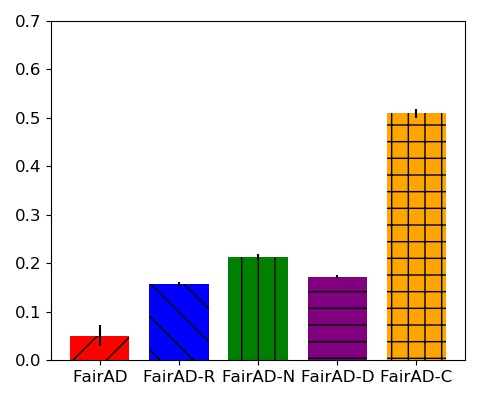}\\
        (a) Recall@1200 &  (b) Recall Difference \\ 
    \end{tabular}
    \caption{Ablation Study on MNIST-USPS dataset.}
    \label{fig:ablation_usps}
    \vspace{-3mm}
\end{figure}


\subsection{Ablation Study (RQ3)}
To validate the necessity of each module in \name, we conduct an ablation study to demonstrate the necessity of each component of \name\ on the MNIST-USPS dataset (more ablation studies can be found in Appendix~\ref{more_ablation_study}). The experimental results are presented in Figure~\ref{fig:ablation_usps}, where (a) and (b) show the recall@1200 and recall difference, respectively. Specifically, \name-R refers to a variant of our method replacing re-balancing autoencoder with $\mathcal{L}_2$ in Equation (\ref{equ:recons}); \name-N and \name-D remove $\mathcal{L}_\mathrm{fair}$ and
$\lunif$ in Equation (\ref{equ:contrastive}), respectively;  \name-C substitutes the proposed fair contrastive loss with the traditional contrastive loss (\ie, $\mathcal{L}_\mathrm{SimCLR}$). We have the following observations. (1) \name\ greatly outperforms \name-D and \name-N, which suggests that $\mathcal{L}_\mathrm{fair}$ and $\lunif$ are two essential components in our designed method. (2) \name-C has the competitive performance as \name\ with respect to recall@1200, but it has a large recall difference. This suggests that without proper regularization, the results exhibit unfair behaviors. Different from \name-C, \name\ achieves a much lower recall difference, which verifies our theoretical analysis that our proposed method could guarantee group fairness. (3) Compared with \name, \name-R has a lower recall rate but a higher recall difference. This indicates that replacing the re-balancing autoencoder with $\mathcal{L}_2$ results in worse performance, which verifies our conjecture that the traditional learning objective tends to mainly focus on learning the frequent patterns of the unprotected group while ignoring the protected group.

\section{Conclusion}
\label{sec:conclusion}
In this paper, we introduce \name, a fairness-aware anomaly detection method, designed for handling the imbalanced data scenario in the context of anomaly detection. Specifically, \name\ maximizes the similarity between the protected and unprotected groups to ensure fairness through the fairness-aware contrastive learning based module. To address the negative impact of imbalanced data, the re-balancing autoencoder module is proposed to reweigh the importance of both the protected and unprotected groups with the learnable weight. Theoretically, we provide the upper bound with Rademacher complexity for the discrepancy between two groups and ensure group fairness through the proposed contrastive learning regularization. Empirical studies demonstrate the effectiveness and efficiency of \name\ across multiple real-world datasets.




\newpage
\appendix
\onecolumn

\section{Notations}
\begin{table}[h!]
\centering
\caption{Notation Table.}
\label{table:symbols}
\begin{tabular}{cc}
\toprule
\textbf{Symbol} & \textbf{Description} \\
\midrule
$x$ & input feature\\
\midrule
$\pp$ & Protected group's distribution\\
\midrule
$\pu$ & Unprotected group's distribution\\
\midrule
$P$ & Protected group's empirical distribution\\
\midrule
$U$ & Unprotected group's empirical distribution\\
\midrule
$n / m$ & Size of protected/unprotected group\\
\midrule
$a_P/ a_U$ & labeling functions on protected/unprotected group\\
\midrule
$\ell$ & Loss function\\
\midrule
$R^\ell_D(h)$ & Risk of hypothesis $h$ over data $D$\\
\midrule
$\hat{R}^\ell_D(h)$ & Empirical risk of hypothesis $h$ over data $D$\\
\midrule
$s(x)$ & Anomaly score of example $x$\\
\midrule
$\mathfrak{R}_D(\mathcal{F}) $ & Rademacher complexity of $\mathcal{F}$ given data $D$\\
\bottomrule
\end{tabular}
\end{table}

\section{Related Work}
\label{related_work}
\textbf{Unsupervised Anomaly Detection.} 
Anomaly detection has been widely studied for decades in many real-world applications, including fraud detection in the finance domain~\cite{DBLP:journals/compsec/WestB16, DBLP:journals/access/HuangMYC18}, pathology analysis in the medical domain~\cite{DBLP:journals/bmcbi/FaustXHGVDD18, DBLP:journals/access/ShvetsovaBFSD21}, intrusion detection for cyber-security~\cite{DBLP:journals/jnca/LiaoLLT13, DBLP:journals/ett/AhmadKSAA21}, and fault detection in safety-critical systems~\cite{DBLP:journals/ijsysc/JuTLM21}, etc. The authors of ~\cite{DBLP:journals/csur/PangSCH21} divide the existing anomaly detection methods into two major branches. The methods~\cite{DBLP:conf/kdd/AudibertMGMZ20, DBLP:conf/icde/ChenDHZZZZ21, DBLP:conf/iccv/HouZZXPZ21, DBLP:conf/aaai/YanZXHH21, DBLP:conf/sdm/WangLMGGLW023} in the first branch aim to learn the patterns for the normal samples by a feature extractor. For instance, \cite{DBLP:conf/kdd/AudibertMGMZ20} is an encoder-decoder anomaly detection method, which learns how to amplify the reconstruction error of anomalies with adversarial training; \cite{DBLP:conf/icde/ChenDHZZZZ21} proposes a GAN-based autoencoder model to learn the normal pattern of multivariate time series, and detect anomalies by selecting the samples with the higher reconstruction error. The second branch aims at learning scalar anomaly scores in an end-to-end fashion~\cite{DBLP:conf/iclr/SohnLYJP21, DBLP:conf/icml/LiCCWTZ23, DBLP:conf/nips/JiangLWNWLWZ22}. Notably, the authors of \cite{DBLP:conf/iclr/SohnLYJP21} combine distribution-augmented contrastive regularization with a one-class classifier to detect anomalies; 
Different from these methods, this paper tackles the problem of fairness-aware anomaly detection by mitigating the representation disparity with contrastive learning-based regularization.

\textbf{Fair Machine Learning.}
Fair Machine Learning aims to amend the biased machine learning models to be fair or invariant regarding specific variables. A surge of research in fair machine learning has been done in the machine learning community\cite{DBLP:conf/kdd/KobrenSM19, zemel2013learning, bolukbasi2016man, hashimoto2018fairness, zhang2018mitigating}. For example, \cite{zemel2013learning} presents a learning algorithm for fair classification by enforcing group fairness and individual fairness in the obtained data representation;  \cite{bolukbasi2016man} proposes approaches to quantify and reduce bias in word embedding vectors that are trained from real-world data; in~\cite{hashimoto2018fairness}, the authors develop a robust optimization framework that minimizes the worst case risk over all distributions and preserves the minority group in an imbalanced data set; in~\cite{zhang2018mitigating}, the authors present an adversarial-learning based framework for mitigating the undesired bias in modern machine learning models. In the field of fair anomaly detection, \cite{DBLP:conf/fat/ZhangD21} utilizes the adversarial generative nets to ensure group fairness and use one-class classification to detect the anomalies; \cite{DBLP:conf/kdd/SongLL21} introduces fairness adversarial training and proposes a novel dynamic weight to reduce the negative impacts from outlier points. The existing fair anomaly detection methods~\cite{DBLP:conf/kdd/SongLL21, DBLP:conf/fat/ZhangD21, DBLP:conf/iccv/FioresiDS23} tend to suffer from the representation disparity issue in the imbalanced data scenario. To address this issue, this paper aims to alleviate the issue of representation disparity in the imbalanced data scenario by introducing the rebalancing autoencoder module and maximizing the uniformity of the samples in the latent space via contrastive learning regularization.

\section{Rademacher Complexity}
\label{appendix:rademacher}
The Rademacher complexity for a function class is:
\begin{definition} (Rademacher Complexity \cite{shalev2014understanding})
Given a space $\mathcal{X}$, and a set of i.i.d. examples D = \(\{x_1, x_2,...,x_{|D|}\} \subseteq \mathcal{X}\), for a function class \(\mathcal{F}\) where each function \(r: \mathcal{X} \rightarrow \mathbb{R}\), the empirical Rademacher complexity of \(\mathcal{F}\) is given by:
\begin{equation}
    \mathfrak{R}_D(\mathcal{F}) = \mathbb{E}_{\sigma}\left[\sup _{r \in \mathcal{F}}\left(\frac{1}{|D|}\sum_{i=1}^{|D|}\sigma_i r(z_i)\right)\right]
\end{equation}
\end{definition}
Here, \(\sigma_1,...,\sigma_m\) are independent random variables uniformly drawn from \(\{-1,1\}\).

\section{Divergences}
\label{appendix:divergence}
We include some popular $f$-divergences in Table \ref{tab:fdivergences}.
\begin{table}[ht]
\centering
\caption{Popular \( f \)-divergences and their conjugate functions.}
\scalebox{0.85}{
\begin{tabular}{ccccc}
\toprule
Divergence &  $f(x)$  & Conjugate $f^*(t)$ &  $f^\prime(1)$ & Activation func. \\
\midrule
Kullback-Leibler (KL) & \( x \log x \) & \( \exp(t - 1) \) & 1 & \( x \) \\
Reverse KL (KL-rev) & \( -\log x \) & \( -1 - \log(-t) \) & -1 & \( -\exp x \) \\
Jensen-Shannon (JS) & \( -{x + 1} \log \frac{1 + x}{2} + x \log x \) & \( -\log(2 - e^t) \) & 0 & \( \log\frac{2}{1+\exp(-x)} \) \\
Pearson \( \chi^2 \) & \( (x - 1)^2 \) & \( \frac{t^2}{4} + t \) & 0 & \( x \) \\
Total Variation (TV) & \( \frac{1}{2} |x - 1| \) & \( 1_{-1/2 \leq t \leq 1/2} \) & \( [-1/2, 1/2] \) & \( \frac{1}{2} \tanh x \) \\
\bottomrule
\end{tabular}}
\label{tab:fdivergences}
\end{table}

\section{Proofs}
\subsection{Proof of Lemma \ref{lemma:epsilon}}
\label{app:prooflemmaepsilon}
Let us divide the data into four types: unprotected normal examples (UN), protected normal examples (PN), unprotected anomalies (UA), and protected anomalies (PA). For type $t \in$ \{UN, PN, UA, PA\},
let $\mathcal{L}_0^t$ denote the loss of the unfitted model on $t$ and $\mathcal{L}_1^t$ as the loss of the fitted model on $t$, $\Delta^t = \mathcal{L}_0^t - \mathcal{L}_1^t > 0$. Assuming that the model can only fit two sets of data, to ensure that the model fits the sets of protected normal examples and unprotected normal examples, we need the following 5 inequalities to hold:

 $(1-\epsilon)(\mathcal{L}^{UN}_1 + \mathcal{L}_0^{UA}) + \epsilon (\mathcal{L}_1^{PN} + \mathcal{L}_0^{PA}) < $

\begin{enumerate}[leftmargin=*]
    \item $(1-\epsilon)(\mathcal{L}^{UN}_0 + \mathcal{L}_1^{UA}) + \epsilon (\mathcal{L}_1^{PN} + \mathcal{L}_0^{PA}) $, implied by $\Delta^{UN} > \Delta^{UA}$
    which naturally holds;
    \item $(1-\epsilon)(\mathcal{L}^{UN}_1 + \mathcal{L}_0^{UA}) + \epsilon (\mathcal{L}_0^{PN} + \mathcal{L}_1^{PA}) $, implied by $\Delta^{PN} > \Delta^{PA}$ which naturally holds;
    \item $(1-\epsilon)(\mathcal{L}^{UN}_0 + \mathcal{L}_1^{UA}) + \epsilon (\mathcal{L}_0^{PN} + \mathcal{L}_1^{PA}) $,  this case is equivalent to case 1 plus case 2;
    \item $(1-\epsilon)(\mathcal{L}^{UN}_1 + \mathcal{L}_1^{UA}) + \epsilon (\mathcal{L}_0^{PN} + \mathcal{L}_0^{PA}) $, we need $\epsilon > \frac{\Delta^{UA}}{\Delta^{UA}+ \Delta^{PN}}$;
    \item $(1-\epsilon)(\mathcal{L}^{UN}_0 + \mathcal{L}_0^{UA}) + \epsilon (\mathcal{L}_1^{PN} + \mathcal{L}_1^{PA}) $, we need $\epsilon < \frac{\Delta^{UN}}{\Delta^{UN}+ \Delta^{PA}}$.
\end{enumerate}
So we have: $ \frac{\Delta^{UA}}{\Delta^{UA}+ \Delta^{PN}} < \epsilon < \frac{\Delta^{UN}}{\Delta^{UN}+ \Delta^{PA}}$. We design $\epsilon = \frac{\mathcal{L}_0^U - \mathcal{L}_U}{\mathcal{L}_0^U - \mathcal{L}_U + \mathcal{L}_0^P - \mathcal{L}_P}$, and we discuss the following three cases:

\begin{itemize}[leftmargin=*]
    \item If $\mathcal{L}_U = \mathcal{L}_1^{UN} + \mathcal{L}_0^{UA}, \mathcal{L}_P = \mathcal{L}_1^{PN} + \mathcal{L}_0^{PA}$, then  $\epsilon = \frac{\Delta^{UN}}{\Delta^{UN} + \Delta^{PN}}$, which is within the range;
    \item If $\mathcal{L}_U = \mathcal{L}_1^{UN} + \mathcal{L}_1^{UA}, \mathcal{L}_P = \mathcal{L}_0^{PN} + \mathcal{L}_0^{PA}$, then  $\epsilon = 1$, it encourages to fit $\mathcal{L}_P$;
    \item If $\mathcal{L}_U = \mathcal{L}_0^{UN} + \mathcal{L}_0^{UA}, \mathcal{L}_P = \mathcal{L}_1^{PN} + \mathcal{L}_1^{PA}$, then  $\epsilon = 0$, it encourages to fit $\mathcal{L}_U$.
\end{itemize}

 We estimate $\mathcal L^U_0 = \sum_{i \in U} \|x_i - \overline{G(x)}\|^2$
 where $\overline{G(x)} = \frac{1}{|U|} \sum_{i \in U} G(x_i) $,  and $\mathcal L^P_0 = \sum_{i \in P} \|x_i - \overline{G(x)}\|^2$ where $\overline{G(x)} = \frac{1}{|P|} \sum_{i \in P} G(x_i) $. Let us denote this as loss1. We also provide results on real-world datasets with different designs of estimation in Table \ref{tab:diffloss}:
 \begin{itemize}[leftmargin=*]
     \item loss2: $\mathcal L^U_0 = \sum_{i \in U} \|x_i \|^2$ and $\mathcal L^P_0 = \sum_{i \in P} \|x_i \|^2$ 
     \item loss3: $\mathcal L^U_0 = \sum_{i \in U} \|G(x_i) - \overline{x}\|^2$ and $\mathcal L^P_0 = \sum_{i \in P} \|G(x_i) - \overline{x}\|^2$
     \item loss4: $\mathcal L^U_0 = \sum_{i \in U} \|x_i - \overline{x}\|^2$ and $\mathcal L^P_0 = \sum_{i \in P} \|x_i - \overline{x}\|^2$ 
 \end{itemize}

 \begin{table*}[t]
    \centering
    \caption{Performance of \name\ with different designs of $\mathcal{L}_0^U$ and $\mathcal{L}_0^P$.}
    \label{tab:diffloss}
    \scalebox{0.7}{
    \begin{tabular}{c|ccc|ccc}
        \toprule
        \multirow{2}{*}{Methods} & \multicolumn{3}{c|}{MNIST-USPS (K=1200)} & \multicolumn{3}{c}{MNIST-Invert (K=500)}\\\cmidrule{2-7}
         & Recall@K & ROCAUC & Rec Diff  & Recall@K & ROCAUC & Rec Diff  \\ \midrule
         
         loss1      & 
         \stat{67.16}{0.37} & \stat{91.27}{0.49} & \stat{3.73}{2.13} & \stat{72.37}{0.32}& \stat{98.03}{0.01}& \stat{6.75}{0.34}\\
         \midrule
         loss2 &\stat{66.47}{1.73}&\stat{90.60}{0.52}& \stat{4.78}{2.36} &\stat{72.44}{0.74}&\stat{98.04}{0.03}&\stat{7.22}{0.21}\\
         \midrule
         loss3 &\stat{66.31}{0.65}&\stat{91.37}{0.88}& \stat{6.32}{1.74} & \stat{71.39}{1.96}& \stat{97.22}{1.42}& \stat{8.95}{0.92}\\
         \midrule
         loss4 &\stat{66.56}{2.32}&\stat{90.88}{1.67}& \stat{2.54}{2.11}& \stat{71.92}{3.58}& \stat{97.01}{1.85}&\stat{8.96}{3.23}\\
        \bottomrule
    \end{tabular}}
\end{table*}
And we can see that although the results may vary with different estimation designs, our method always performs better than the baselines in both task performance and fairness.
 
\subsection{Proof of Lemma \ref{lemma:df}}
\label{sec:allproofs}
\begin{equation*}
\begin{aligned}
D^{f}_{h,\mathcal{H}} (P_U\|P_P) - D^{f}_{h,\mathcal{H}}(U\|P) &= \sup_{h' \in \mathcal{H}} \{ |R^{\ell}_{U}(h, h') - R^{f^* \circ \ell}_{P}(h, h')| \}  \\
&\quad - \sup_{h' \in \mathcal{H}} \{ |\hat{R}^{\ell}_{U}(h, h') - \hat{R}^{f^* \circ \ell}_{P}(h, h')| \} \\
&\leq \sup_{h' \in \mathcal{H}} \|R^{\ell}_{U}(h, h') - R^{f^* \circ \ell}_{P}(h, h')| - |\hat{R}^{\ell}_{U}(h, h') - \hat{R}^{f^* \circ \ell}_{P}(h, h')\| \\
&\leq \sup_{h' \in \mathcal{H}} |R^{\ell}_{U}(h, h') - R^{f^* \circ \ell}_{P}(h, h') - \hat{R}^{\ell}_{U}(h, h') + \hat{R}^{f^* \circ \ell}_{P}(h, h')| \\
&= \sup_{h' \in \mathcal{H}} |R^{\ell}_{U}(h, h') - \hat{R}^{\ell}_{U}(h, h') |+ |{R}^{f^* \circ \ell}_{P}(h, h') - \hat{R}^{f^* \circ \ell}_{P}(h, h') | \\
&\leq 2\mathfrak{R}_{P_U}(\ell \circ \mathcal{H}) + \sqrt{\frac{\log \frac{1}{\delta}}{2m}} + 2\mathfrak{R}_{P_P}(f^* \circ \ell \circ \mathcal{H}) + \sqrt{\frac{\log \frac{1}{\delta}}{2n}}
\end{aligned}
\end{equation*}
where the last inequality comes from the property of Rademacher complexity. Similarly, by Lemma 5.7 and Definition 3.2 of \cite{mohri2018foundations} we have: $\mathfrak{R}_{P_P}(f^* \circ \ell \circ \mathcal{H}) \leq L\mathfrak{R}_{P_P}(\ell \circ \mathcal{H})$, with $f^* \circ \ell \circ \mathcal{H} := \{ x \mapsto \phi(\ell(h(x), h'(x))): h, h' \in \mathcal{H} \}$.

\subsection{Proof of Theorem \ref{theorem:riskdiff}}
\label{app:proof45}
\begin{align*}
R^\ell_P(h, a_P) &\leq R^\ell_P(h, h^*) + R^\ell_P(h^*, a_P) & (\text{triangle inequality } \ell) \\
&= R^\ell_P(h, h^*) + R^\ell_P(h^*, a_P) - R^\ell_U(h, h^*) + R^\ell_U(h, h^*) \\
& \leq R^{f^* \circ \ell}_P(h, h^*) - R^\ell_U(h, h^*) + R^\ell_U(h, h^*) + R^\ell_P(h^*, a_P)\\
& \leq | R^{f^* \circ \ell}_P(h, h^*) - R^\ell_U(h, h^*) | + R^\ell_U(h, h^*) + R^\ell_P(h^*, a_P)\\
&\leq D^{f}_{h,\mathcal{H}}(P_U\|P_P) + R^\ell_U(h, h^*) + R^\ell_P(h^*, a_P) \\
&\leq D^{f}_{h,\mathcal{H}}(P_U\|P_P) + R^\ell_U(h, a_U)+ R^\ell_U(h^*, a_U) + R^\ell_P(h^*, a_P)\\
& = D^{f}_{h,\mathcal{H}}(P_U\|P_P) + R^\ell_U(h)+ R^\ell_U(h^*) + R^\ell_P(h^*)
\end{align*}

\subsection{Proof of Theorem \ref{theorem:bound} and the benefit of our design}
\label{appendix:fairnessbound}

Combining Theorem \ref{theorem:riskdiff}, Lemma \ref{lemma:df} and the property of Rademacher Complexity, we can easily get:

\begin{equation*}
\begin{aligned}
 R^l_P(h) &-  R^l_U(h)\leq  D^f_{h,\mathcal{H}}(U \| P)\\
& + \hat{R}^l_U(h^*) + 4\mathfrak{R}_U(\ell \circ \mathcal{H}) + 2 \sqrt{\frac{\log \frac{1}{\delta}}{2m}} \\
& + \hat{R}^l_P(h^*) + 2 (L+1) \mathfrak{R}_P(\ell \circ \mathcal{H}) + 2 \sqrt{\frac{\log \frac{1}{\delta}}{2n}}
 \end{aligned}
\end{equation*}

Since by definition we have $ D^f_{h,\mathcal{H}}(U \| P) \leq D_f(U\|P)$, and for $ D_f(U\|P)= \TV(U\|P) $, we have:
\begin{equation}
\begin{aligned}
 R^l_P(h) &-  R^l_U(h)\leq  \TV(U \| P)\\
& + \hat{R}^l_U(h^*) + 4\mathfrak{R}_U(\ell \circ \mathcal{H}) + 2 \sqrt{\frac{\log \frac{1}{\delta}}{2m}} \\
& + \hat{R}^l_P(h^*) + 2 (L+1) \mathfrak{R}_P(\ell \circ \mathcal{H}) + 2 \sqrt{\frac{\log \frac{1}{\delta}}{2n}}.
 \end{aligned}
\end{equation}

\newcommand{\X}{X}
\newcommand{\Y}{Y}
\renewcommand{\P}{P}
\newcommand{\U}{U}
\renewcommand{\r}{r}
\newcommand{\m}{x^*}
\newcommand{\z}{z}
\newcommand{\w}{w}
\newcommand{\eps}{\varepsilon}
\newcommand{\support}{\mathcal{X}}
\newcommand{\hP}{\hat{P}}
\newcommand{\hp}{\hat{p}}
\newcommand{\cx}{c_X}
\newcommand{\cy}{c_Y}
\newcommand{\cu}{c_U}
\newcommand{\cp}{c_P}
\newcommand{\x}{x}

Now we motivate why minimizing the objective $\lcon$ leads to small $\TV(U\|P)$. Let $U,P$ be the empirical distributions over the common measurable space $\support\coloneqq\{z_j^U\}_{j=1}^n\cup\{z_k^P\}_{k=1}^m$ with densities $\hp_{\U},\hp_{\P}$ that are $\cu,\cp$-Lipschitz with respect to $\ell_2$-norm, respectively. Let $\m\coloneqq\argmin_{\x\in\support}\abs{\hp_{\U}(\x)-\hp_{\P}(\x)}$, $\delta\coloneqq\abs{\hp_{\U}(\m)-\hp_{\P}(\m)}$, and 
\[
\sigma\coloneqq \sum_{\x\in\support}\norm{\x-\m}=\sum_{\x\in\support}\sqrt{2-2\log\cossim(\x,\m)}, 
\]
where the equality is due to law of cosine (and that $\cossim$ normalizes $z_j$). We first show how $\TV(U\|P)$ is related to $\delta$ and $\sigma$.
\begin{lemma}
\label{lem-tv1}
\[
\TV\left(U\|P\right)\le \frac{1}{2}\left(\abs{\support}\delta+(\cu+\cp)\sigma\right).
\]
\end{lemma}
\begin{proof}
\begin{align*}
\TV\left(U\|P\right) &\coloneqq \frac12\sum_{\x\in\support}\abs{\hp_{\U}(\x)-\hp_{\P}(\x)} \\
&\le\frac12\sum_{\x\in\support}\abs{\hp_{\U}(\x)-\hp_{\U}(\m)}+\abs{\hp_{\U}(\m)-\hp_{\P}(\m)}+\abs{\hp_{\P}(\m)-\hp_{\P}(\x)} && \text{(triangle inequality)}\\
&=\frac12\left(\abs{\support}\delta+\sum_{\x\in\support}\abs{\hp_{\U}(\x)-\hp_{\U}(\m)}+\abs{\hp_{\P}(\x)-\hp_{\P}(\m)}\right)\\
&\le\frac12\left(\abs{\support}\delta+ (\cu+\cp)\sum_{\x\in\support}\norm{\x-\m}\right) && \text{(Lipschitz conditions)}\\
&=\frac{1}{2}\left(\abs{\support}\delta+(\cu+\cp)\sigma\right).
\end{align*}
\end{proof}

Next we motivate why minimizing our objective $\lcon$ leads to small $\delta$ and $\sigma$ simultaneously, hence small $\TV(U\|P)$. Recall that our fairness-aware contrastive loss is 
\[
\lcon \coloneqq \mathcal{L}_{\mathrm{fair}}+ \lunif,
\]
where 
\begin{align*}
\mathcal{L}_{\mathrm{fair}} &\coloneqq -\log\left(\sum_{j\in[n]}\sum_{k\in[m]}\cossim\left(z_j^U, z_k^P\right)\right),\\
\lunif &\coloneqq \log\left(\sum_{j\ne k} \cossim\left(z_j^U, z_k^U\right) +  \sum_{j\ne k} \cossim\left(z_j^P, z_k^P\right)\right).
\end{align*}
Intuitively, minimizing $\lcon$ leads to small $\lfair$ and $\lunif$ simultaneously, which correspond to large $\cossim(z_j^U,z_k^P)$ and small $\cossim(z_j^U,z_k^U),\cossim(z_j^P,z_k^P)$, which in turn correspond to small $\norm{z_j^U-z_k^P}$ and large $\norm{z_j^U-z_k^U},\norm{z_j^P-z_k^P}$. Hence it is natural to consider the following surrogate losses
\begin{align*}
\surfair&\coloneqq\sum_{j, k\in[n]}\norm{z_j^U-z_k^P},\\
\surunif&\coloneqq-(\sum_{j\ne k}\norm{z_j^U-z_k^U} +  \norm{z_j^P-z_k^P}).
\end{align*}
Then it follows immediately that $\sigma\le\surfair$, explaining why minimizing our objective $\lcon$ (hence $\surfair$) leads to small $\sigma$.

To see that $\delta\coloneqq\abs{\hp_{\U}(\m)-\hp_{\P}(\m)}$ cannot be too large, first consider the extreme case where $\{z_j^U\}_{j=1}^n\cap\{z_k^P\}_{k=1}^n=\emptyset$. Without loss of generality let $\norm{z_1^U-z_1^P}=\max_{j,k\in[n]}\norm{z_j^U-z_k^P}$. Then adjusting $z_1^U,z_1^P$ to be the unit vector on their angle bisector clearly decreases $\surfair$ without affecting $\surunif$ by much due to high uniformity within $\{z_j^U\}_{j=1}^n$ and $\{z_k^P\}_{k=1}^n$ respectively. Hence we may assume without loss of generality that $z_1^U=z_1^P=\m$. Next consider the extreme case where $\hp_{\U}(\m)=\frac1n$ and $\hp_{\P}(\m)=1$. Then adjusting $z_2^P=\argmax_{x\ne\m}\sum_{j\in[n]}\norm{x-z_j^U}$ clearly decreases $\surunif$ without affecting $\surfair$ by much due to high uniformity within $\{z_j^U\}_{j=1}^n$. Hence minimizing our objecive $\lcon$ leads to small $\delta\coloneqq\abs{\hp_{\U}(\m)-\hp_{\P}(\m)}$.

\section{Additional Experiments}
\subsection{Training details and Code}
\label{app:code}
For the COMPAS dataset, we use a two-layer MLP with hidden units of [32, 32]. For all the other datasets, we use MLP with one hidden layer of dimension 128.  All our experiments were executed using one Tesla V100 SXM2 GPUs, supported by a 12-core CPU operating at 2.2GHz. For three runs, we choose random seeds in [40,41,42].


\subsection{More effectiveness validation of \name\ under different k}
\label{more_recall_k}
We also conduct experiments on the four datasets with different choices of k, and the results are in Table \ref{tab:image_datasets_2} and Table \ref{tab:tabular_datasets_2}. The AUCROC scores are the same as in the main paper. We can also tell from the tables that accuracy difference is inadequate for measuring group fairness in the imbalanced setting. 
\begin{table*}[ht]
    \centering
    \caption{Performance on Image Datasets. }
    \scalebox{0.8}{
    \begin{tabular}{c|ccc|ccc}
        \toprule
        \multirow{2}{*}{Methods} & \multicolumn{3}{c|}{MNIST-USPS (K=1000)} & \multicolumn{3}{c}{MNIST-Invert (K=400)}\\\cmidrule{2-7}
         & Recall@K & Acc Diff& Rec Diff  & Recall@K & Acc Diff & Rec Diff   \\ \midrule
         FairOD      & \stat{10.46}{1.16} & \stat{4.35}{0.33} & \stat{13.21}{1.43} & \stat{6.05}{0.21}   & \stat{2.70}{0.15} & \stat{9.99}{1.18} \\ \midrule
         DCFOD       & \stat{10.24}{0.82} & \stat{4.79}{1.12} & \stat{8.40}{1.83} & \stat{5.57}{1.70}   & \stat{2.69}{0.37} & \stat{8.78}{2.31} \\ \midrule
         FairSVDD   &  \stat{13.75}{1.83} & \stat{5.73}{5.64} & \stat{13.49}{2.55} & \stat{10.57}{0.92}   & \stat{5.38}{3.12} & \stat{14.25}{2.96} \\  \midrule
         MCM        &  \stat{34.38}{0.32} & \stat{29.81}{0.84} & \stat{52.46}{0.94} & \stat{22.48}{0.54}   & \stat{8.32}{1.10} & \stat{64.37}{1.66} \\  \midrule
         NSNMF      & \stat{33.56}{0.70} & \stat{22.26}{0.40} & \stat{65.12}{2.36} & \stat{43.91}{0.84}   & \stat{4.54}{0.20} & \stat{55.20}{0.92}  \\  \midrule
         Recontrast   & \stat{45.73}{2.74} & \stat{10.59}{2.62} & \stat{29.62}{2.40} & \stat{52.00}{4.86} & \stat{13.81}{4.30} & \stat{54.96}{13.77}  \\ \midrule
         \name      & \stat{61.60}{2.50} & \stat{6.50}{0.89} & \stat{7.95}{5.94} & \stat{62.28}{3.24}   & \stat{1.62}{1.32} & \stat{7.02}{4.48}  \\ 
        \bottomrule
    \end{tabular}}
    
    \label{tab:image_datasets_2}
\end{table*}

\begin{table*}[ht]
    \centering
    \caption{Performance on Tabular Datasets}
    \scalebox{0.8}{
    \begin{tabular}{c|ccc|ccc}
        \toprule
        \multirow{2}{*}{Methods} & \multicolumn{3}{c|}{COMPAS (K=300)} & \multicolumn{3}{c}{CelebA (K=4500)}\\\cmidrule{2-7}
         & Recall@K & Acc Diff& Rec Diff  & Recall@K & Acc Diff & Rec Diff  \\ \midrule
         FairOD      & \stat{14.20}{1.83} & \stat{3.92}{1.63} & \stat{10.75}{0.90} & \stat{7.95}{0.21}   & \stat{4.94}{0.25} & \stat{2.26}{1.06}  \\  \midrule
         DCFOD       & \stat{13.10}{1.35} & \stat{3.57}{2.29} & \stat{7.23}{2.82} & \stat{8.64}{0.79}   & \stat{4.98}{0.40} & \stat{9.24}{1.12}   \\  \midrule
         FairSVDD    & \stat{13.02}{1.66} & \stat{3.90}{2.43} & \stat{9.45}{3.80} & \stat{8.82}{0.61}   & \stat{2.21}{0.40} & \stat{10.22}{2.33} \\   \midrule
         MCM         & \stat{16.87}{1.14} & \stat{4.10}{1.98} & \stat{10.17}{1.64} & \stat{9.26}{0.48}   & \stat{7.21}{5.98} & \stat{28.69}{12.14}   \\   \midrule
         NSNMF       & \stat{17.29}{1.42} & \stat{3.60}{1.93} & \stat{33.57}{1.22} & \stat{8.90}{1.09}   & \stat{5.66}{0.54} & \stat{40.51}{1.54} \\   \midrule
         \name       & \stat{19.14}{2.29} & \stat{9.35}{3.00} & \stat{4.75}{3.69} & \stat{10.56}{1.11}   & \stat{13.04}{0.30} & \stat{5.10}{1.52}   \\ 
        \bottomrule
    \end{tabular}}
    \label{tab:tabular_datasets_2}
\end{table*}

\subsection{More Ablation Study}
\label{more_ablation_study}
We conduct the ablation study to demonstrate the necessity of each component of \name\ on the Compas dataset. The experimental results are presented in Figure~\ref{fig:ablation_compas}, where (a) and (b) show the recall@350 and recall difference, respectively. Specifically, \name-R refers to a variant of our method replacing rebalancing autoencoder with $\mathcal{L}_2$ in Equation (\ref{equ:recons}); \name-N and \name-D remove $\mathcal{L}_\text{fair}$ and $\mathcal{L}_\text{unif}$ in Equation (\ref{equ:contrastive}), respectively;  \name-C substitute the proposed fair contrastive loss with the traditional contrastive loss (\ie, $\mathcal{L}_1$). We have the following observations: (1) \name\ greatly outperforms \name-D and \name-N, which suggests that $\mathcal{L}_\text{fair}$ and $\mathcal{L}_\text{unif}$ are two essential components in our designed method. (2). \name-C has the competitive performance as \name\ with respect to recall rate, while it has a large recall difference. This suggests that without proper regularization, the results exhibit unfair behaviors. Different from \name-C, \name\ achieves a much lower recall difference, which further verifies our assumption that our proposed method could guarantee group fairness.

\begin{figure}[h]
    \centering
    \begin{tabular}{cc}
         \includegraphics[width=0.25\linewidth]{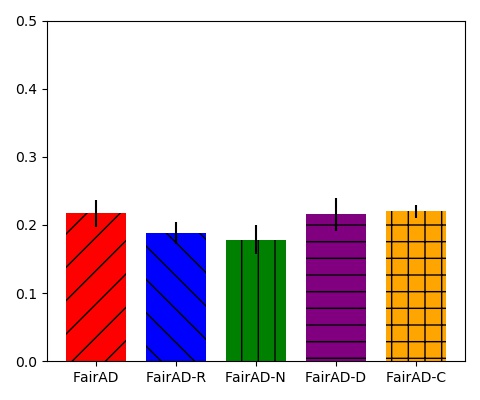} & \includegraphics[width=0.25\linewidth]{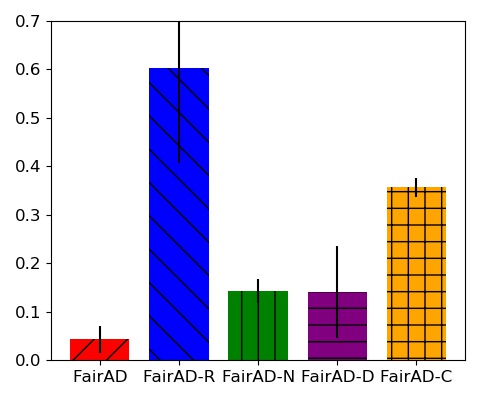}\\
        (a) Recall@350 &  (b) Recall Difference \\ 
    \end{tabular}
    \caption{Ablation Study on Compas dataset.}
    \label{fig:ablation_compas}
\end{figure}

\section{Limitations and Broader Impact}
This paper proposes a fairness-aware anomaly detection, which aims to provide fair results when the algorithm is applied to detect anomalies. Our method currently focus on the binary group fairness case. We can naturally extend our framework to the multi-attribute case by encouraging the similarity among the groups. Incoporating individual fairness notions would be an interesting future direction. By embedding fairness into anomaly detection algorithms, this work contributes to reducing bias and discrimination in AI applications, ensuring that technologies serve diverse populations equitably. In sectors such as finance, healthcare, and law enforcement, where anomaly detection plays a crucial role in identifying fraud, diseases, and criminal activities, incorporating fairness principles can prevent the perpetuation of historical biases and protect vulnerable groups from unjust outcomes. Furthermore, by advancing fairness in AI, this research aligns with global efforts to promote ethics in technology development, fostering trust between AI systems and their users.

\end{document}